\documentclass{article}

\usepackage[nonatbib,preprint]{neurips_2023}

\usepackage[utf8]{inputenc} % allow utf-8 input
\usepackage[T1]{fontenc}    % use 8-bit T1 fonts
\usepackage{hyperref}       % hyperlinks
\usepackage{url}            % simple URL typesetting
\usepackage{booktabs}       % professional-quality tables
\usepackage{amsfonts}       % blackboard math symbols
\usepackage{nicefrac}       % compact symbols for 1/2, etc.
\usepackage{microtype}      % microtypography
\usepackage{xcolor}         % colors

\usepackage{amsmath}
\usepackage{amssymb}
\usepackage{mathtools}
\usepackage{amsthm}
\usepackage{mathrsfs}
\newtheorem{theorem}{Theorem}[section]
\newtheorem{proposition}[theorem]{Proposition}

\newtheorem{definition}[theorem]{Definition}
\newtheorem{assumption}[theorem]{Assumption}
\newtheorem{remark}[theorem]{Remark}

\newcommand{\equref}[1]{Eq.~(\ref{#1})}
\newcommand{\figref}[1]{Fig.~\ref{#1}}
\newcommand{\secref}[1]{Sec.~\ref{#1}}
\newcommand{\tabref}[1]{Tab.~\ref{#1}}

\title{Plug-and-Play Algorithm Convergence Analysis From The Standpoint
of Stochastic Differential Equation}

\author{%
Zhongqi Wang,~ Bingnan Wang,~ Maosheng Xiang\\
  School of Electronic Electrical and Communication Engineering\\ 
  University of Chinese Academy of Science\\
  Beijing, China\\
\texttt{wangzhongqi20@mails.ucas.ac.cn,~wbn@mail.ie.ac.cn,~xms@mail.ie.ac.cn} \\
}

\begin{document}

\maketitle

\begin{abstract}
% ideally between 4--6 sentences long.

The Plug-and-Play (PnP) algorithm is popular for inverse image problem-solving. 
However, this algorithm lacks theoretical analysis of its convergence with 
more advanced plug-in denoisers. 
We demonstrate that discrete PnP iteration can be described by a continuous 
stochastic differential equation (SDE). 
We can also achieve this transformation through Markov process formulation of PnP. 
Then, we can take a higher standpoint of PnP algorithms from 
stochastic differential equations, 
and give a unified framework for the convergence property of PnP according to the 
solvability condition of its corresponding SDE. 
We reveal that a much weaker condition, bounded denoiser with Lipschitz continuous 
measurement function would be enough for its convergence guarantee, instead of 
previous Lipschitz continuous denoiser condition. 

\end{abstract}

\section{Introduction}

The inverse problem emerges in various fields, such as 
astrophysics \cite{2018Inverse_Problems_Astronomy}, 
computational imaging \cite{2015Bouman_MBIP}, 
and numerical computation \cite{2017InverseProblemDifferentialEquation}. 
Consider the inverse problem with a given measurement system $M$, 
\begin{equation}
    y = Mx+\epsilon.  
\end{equation}
We meant to compute the target $x$ according to the acquired signal $y$ and eliminate perturbation $\epsilon$. The challenge would be the irreversible of 
the measurement model $M$. There could be 
information loss due to the 
insufficient measurement and the perturbations, which makes this an underdetermined problem. 

\cite{2013GlobalSIP_pnp} proposed the Plug-and-Play (PnP) methods to solve inverse image problems, such 
as image denoising, image inpainting, demosaic, super-resolution, phase retrieval, 
microscopy image reconstruction, etc. In image processing, a general 
form of inverse problem solving would be: 
\begin{equation}
    \mathop{\min}\limits_{x} f(x,y)+\gamma g(x), 
\end{equation}
where $x$ is the image we would like to reconstruct, $f(x,y)=y-Mx$ is the data fidelity term and 
$g(x)$ are the constraints of image properties, usually the qualification of smoothness . 
$\gamma$ is a parameter that controls the weight of $g$. The above formulation 
can also be interpreted as Maximum A Posteriori problems. Given the posterior probability 
\begin{equation}
    P(x|y) = \frac{P(y|x)P(x)}{P(y)}, 
\end{equation}
we seek to compute the optimal $\hat{x}$ that maximizes the conditional distribution
\begin{align}
    % \hat{x} &= \mathop{\arg\max}\limits_{x} P(x|y) \\
    % &= \mathop{\arg\min}\limits_{x} -\ln P(y|x) - \ln P(x).  
    \hat{x} = \mathop{\arg\min}\limits_{x} -\ln P(y|x) - \ln P(x). 
\end{align}
with $f(x)=-\ln P(y|x)$ being the likelihood function and $g(x)=- \ln P(x)$ being the 
prior function.

For solving this joint optimization problem, the first step is to split the variable $x$ 
by adding an auxiliary variable $v$
\begin{equation}
    (\hat{x}, \hat{v}) = \mathop{\arg\min}\limits_{x=v} f(x,y)+\gamma g(v)
\end{equation}
To handle this constrained optimization problem \cite{2010FTML_ADMM}, 
the augmented lagrangian would be 
\begin{equation}
    L_{\lambda}(x,v;u)=f(x,y)+\gamma g(v) + \frac{1}{2\lambda^2}\|x-v+u\|_2^2 - \frac{\|u\|_2^2}{2\lambda^2}, 
\end{equation}
where $\lambda>0$ is the augmented Lagrangian parameter.

There are plenty of variable splitting algorithms, like 
ADMM (Alternative Direction Method of Multipliers) \cite{2010FTML_ADMM}, 
HQS (Half-Quadratic Splitting) \cite{2011ICCV_HQS}, 
FIST (Fast Iterative Soft Thresholding) \cite{2017LSP_FIST}, etc. Consider the ADMM algorithm, which consists of iteration steps: 
\begin{align}
    \hat{x} &\leftarrow \mathop{\arg\min}\limits_{x\in\mathbb{R}^N} L_{\lambda}(x,\hat{v};u)  \\
    \hat{v} &\leftarrow \mathop{\arg\min}\limits_{v\in\mathbb{R}^N} L_{\lambda}(\hat{x},v;u)  \\
          u &\leftarrow u+(\hat{x}-\hat{v}), 
\end{align}
ADMM also gets a global convergence guarantee escorted if the two functions 
$f(x)$ and $g(x)$ are 
both proper, closed, and convex and a saddle point solution also exists. 

Hereafter, we would like to express the ADMM iterations as two operators. 
\begin{align}
    F(\tilde{x};\lambda) = \mathop{\arg\min}\limits_{x\in\mathbb{R}^N} \{f(x)+\frac{\|x-\tilde{x}\|_2^2}{2\lambda^2}\}  \\
    G(\tilde{v};\sigma) = \mathop{\arg\min}\limits_{v\in\mathbb{R}^N} \{\frac{\|\tilde{v}-v\|_2^2}{2\sigma^2}+g(v)\}
\end{align}
where the operator $F$ corresponds to the inversion of the measurement model, the denoising 
operator $G$ corresponds to the prior model of images and the parameter $\sigma=\sqrt{\gamma}\lambda$ 
can be interpreted as the assumed noise standard deviation. So $G$ can be conducted 
completely by an off-the-shelf denoiser \cite{2013GlobalSIP_pnp}. In our opinion, the very progress that the Plug-and-Play algorithm has made is the separation of these two operators.

With much more advanced denoisers employed, this algorithm obtains state-of-the-art 
performance \cite{2017cvpr_KaiZhang_DnCNN,2021PAMI_kaizhang_DPIR} that 
even comparable with end-to-end deep learning methods, while preserves 
the advantage of applying to multiple tasks. 
However, we can not acquire a corresponding prior model expression for the given  
plug-in denoiser. Therefore, Plug-and-Play algorithms would not inherit the well-established 
convergence properties of ADMM algorithms \cite{2010TIP_ADMM_convergence}. 
% While this algorithm maintains 
% good performances, its origination is quite heuristic and lacks theoretical support. 

There are many following theoretical research for its convergence analysis 
\cite{2016TCI_pnp_nonexpansive,2016TCI_pnp_bounded,2017SIAM_RED}. 
Most of them continue to pave along the fixed-point type methodology. Briefly speaking, they 
mainly relies on nonexpansive or Lipschitz continuous denoisers assumptions to construct a Banach contraction 
operator, or form a Cauchy convergence sequence equally.
However, most denoisers do not obey 
this seemingly simple assumption, they would not give a reasonable theoretical guarantee 
for the convergence of the practical PnP algorithm. 

To bridge the gap between practical plug-in denoiser and theoretical analysis, 
we propose to describe PnP iteration by a continuous stochastic differential equation. 
This view endows us with a higher standpoint so that we can discuss PnP 
convergence property via SDE solvability. 
% \nocite{2016Elementary}
The main contributions are as follows: 
\begin{itemize}
    \item We give the SDE description of the PnP algorithm. 
    Furthermore, we show two approaches to realize this transformation 
    between PnP and SDE. 
    \item Then, we constructed a unified framework for PnP convergence analysis 
    through the solvability of the corresponding SDE description of the PnP. 
    \item Accordingly, we propose a much weaker condition, 
    which is more applicable to current advanced denoiser priors without additional 
    constrains. 
\end{itemize}

\section{SDE description of PnP}

% The primary inspiration of our proposition is from a score-based generative model through 
% stochastic differential equation, 
Continuous descriptions of discrete algorithms could facilitate the decoupling of 
theoretical analysis and technical implementation. We will see in the next 
section that convergence analysis may take advantage of continuous form. Here, 
we give a direct transformation between a simplified version of PnP and SDE, and two 
other possible pathways to link residual Gaussian denoiser to the pure stochastic 
part of Ito-type SDE. 
% three pathways that empower the conversion from discrete PnP 
% iteration to stochastic differential equation. 
% In this way, we would like to display the solid background of this transformation. 
% Also, these methods intersect, which may boost 
% further research. 

\subsection{Simplified PnP Problem Formulation}

\begin{definition}[PnP iteration]\label{simplified_pnp_iteration}
    In a simplified PnP framework without auxillary variable $u=u+(x-v)$, just the 
    two step iterations: 

    \begin{align}
        {x^{t+1}} &= \mathop{\arg\min}\limits_{x\in\mathbb{R}^N} \{f(x)+\frac{\|x-\tilde{x}\|_2^2}{2\lambda^2}\} 
        = h(v^t, y) \label{eq:first_step} \\
        {v^{t+1}} &= \mathop{\arg\min}\limits_{x\in\mathbb{R}^N} \{\frac{\|x^{t+1}-x^t\|_2^2}{2\sigma_t^2}-\log P(x)\}  
        =D(x^{t+1}, \sigma_t)   \label{eq:denoising_prior}
    \end{align}
    Or a more succinct iteration mapping: 
    \begin{align}
        v^{t+1}=D(h(v^t;y), \sigma_t)
    \end{align}
\end{definition}

The first step \equref{eq:first_step} is a close form solution function $h(\cdot)$ as long as
$f(\cdot)$ is differentiable, this step may operate 
the inverse mapping of image restoration likes uperresolution, inpainting, demosic, etc. 
The second step \equref{eq:denoising_prior} is one step denoising $D(\cdot)$. 
The variance sequence $\sigma_t$ are settled as hyperparameters beforehand, 
the adjust strategy may be linear decay, exponential decay, etc. 

Note that the reason why the proximal mapping in the second step \equref{eq:denoising_prior} 
can be considered as a Gaussian denoising process is that  
$\mathop{\arg\min}\limits_{x\in\mathbb{R}^N} \{\frac{\|x^{t+1}-x^t\|_2^2}{2\sigma_t^2}-\log P(x)\}$ 
resembles to a Maximum A Posteriori(MAP) problem of Gaussian denoising, 
with measurement function being 
$M(x^{t+1}, x^t, \sigma_t)=\frac{\|x^{t+1}-x^t\|_2^2}{2\sigma_t^2}$, 
where the expectation is $x^t$ and the variance is $\sigma_t$. 
And this Gaussian denoising analogy firstly appeared in \cite{2013GlobalSIP_pnp}, 
which is the origination of Plug-and-Play methods. 
So this is a solid evidance to support the Gaussian noise assumption. 

\begin{assumption}[Gaussian noise in PnP denoising step]\label{Gaussian_noise_assumption}
    The inversed transition kernel of Plug-and-Play denoising step is assumed to follow a 
    Gaussian distribution. 
    \begin{align}
        P(x^{t+1}|v^{t+1}) \sim \mathcal{N}( v^{t+1} ,\sigma_{t}). 
    \end{align}
\end{assumption}

Similarly, \cite{2019TCI_scoreMatching_RED_PnP} also adopted the Gaussian denoising in the 
score matching paradigm to interpret RED (Regularization by Denoising) 
method\cite{2017SIAM_RED}, 
which is one of the PnP (Plug-and-Play) variants. 
Denoising Score Matching gives the Gaussian noise assumption for a single step 
% \begin{equation}
%     \frac{1}{\sigma^2} (D(x)-x) \sim \nabla_x \log q_{\sigma}(x)
% \end{equation}
% which is also known as the tweedie's theorem. The denoiser should be an MMSE denoiser.
% So the residual of denoising as in the left hand side approximates the gradient of log 
% density function(score function) in the right hand side. For Gaussian denoising problem 
% the density function would be the Gaussian distribution. 
% So this link between Denoiser and Socre-Matching is meant to 
% give a different interpretion of denoiser as pure Gaussian process with diffusion intensity 
% $\sigma$. 
% For more details, please refer to the denoising score-matching\cite{2011_connection_denoising_score_matching}. 
% For example, 
as equation (10) in \cite{2011_connection_denoising_score_matching}:
\begin{equation}
    \frac{\partial \log q_{\sigma}(\tilde{x}|x)}{\partial \tilde{x}} = \frac{1}{\sigma^2} (x-\tilde{x})
\end{equation}
This equation shows that the stein's score function (i.e. gradient of log density) 
equals to the Gaussian denoising.

% \subsection{Gaussian Noise Assumption}\label{Gaussian Noise Assumption}
% Before diving into the transformation methodology between PnP and SDE, we would like 
% to emphasis the Gaussian noise assumption first. 
% In general, the image denoising process is not always Gaussian. But the denoising 
% process derived in the PnP step \equref{eq:denoising_prior} would owe to the Gaussian 
% denoising task analogy for the MAP problem therein. 
% % but it depends on the noise rather than the algorithm. 
% % So we refer to the training assumption on noise instead of the 
% % inference assumption of denoising process.

% Previous 
% theoretical convergence analysis of PnP rarely mentioned this assumption, 
% except in \cite{2019TCI_scoreMatching_RED_PnP}. 
% This assumption will be important in \secref{Inversed PnP-ADMM} for interpreting denoiser 
% as standard Gaussian process in 
% Ito-type SDE. 

% In practice, the denoiser used in PnP framework, like DnCNN, 
% would be trained by constructed Gaussian noised dataset. 
% And the variance would change to adapt to different noise intensity. So there is an 
% extra controlling variable as an extra channel attached to the input noised image, 
% which is indeedly the variance. 

% So the trained denoiser should be defined as:
% \begin{align} D(I, \sigma)-I \sim \mathcal{N}(0, \sigma^2)\end{align}
% where the denoiser $D(\cdot)$ takes the noised image $I$ and the noise 
% variance $\sigma$ as the inputs. The residual 
% as the left hand side was trained to approximate the right hand side 
% Gaussian distribution. 

\subsection{Main Theorem}\label{Inversed PnP-ADMM}

% The Gaussian assumption of the denoiser has been established in \secref{Gaussian Noise Assumption} as above. 
% In a simplified PnP framework without auxillary variable $u=u+(x-v)$, just the 
% two step iterations: 

% \begin{align}
%     {x^{t+1}} &= h(v^t, y)  \\
%     {v^{t+1}} &= \mathop{\arg\min}\limits_{x\in\mathbb{R}^N} \{\frac{\|x^{t+1}-x^t\|_2^2}{2\sigma_t^2}-\log P(x)\}   \\
%     &=D(x^{t+1}, \sigma_t) 
% \end{align}

% The first step is a close form solution function $h(\cdot)$, may operate 
% the inverse mapping of image restoration likes uperresolution, inpainting, demosic, etc. And 
% the second step is denoising. 
% The variance sequence $\sigma_t$ are settled as hyperparameters beforehand, 
% the adjust strategy may be linear decay, exponential decay, etc. 

\begin{theorem}[SDE description of PnP]\label{main_theorem}
    The discrete iteration of PnP algorithm in \ref{simplified_pnp_iteration}
    can be described by a continuous SDE as: 
    \begin{align}
        dv_t = b(t,v_t)dt - \sigma(t,v_t)dW_t. 
    \end{align}
    Where $b(x, y)=h(x,y)-x$ is the drift term of the PnP-SDE, the diffusion term $\sigma(t)$ 
    only depends on time $t$, which is exactly the variance parameter $\sigma_t$ of 
    the Gaussian denoiser $D_\sigma(\cdot)$ of PnP. 
\end{theorem}

\begin{proof}
    The difference of PnP iteration steps in definition \ref{simplified_pnp_iteration} will be:
    \begin{align}
        {x^{t+1}}-v^t &= b(v^t, y)  \\
        {v^{t+1}}-x^{t+1} &= D(x^{t+1}, \sigma_t)-x^{t+1} \label{eq:denoiser}
    \end{align}
    Where the function $b(v^t,y)=h(v^t,y)-v^t$ for the first step. And in the second 
    step, $D(x^{t+1}, \sigma_t)-x^{t+1}$ is the one step denoising operation, and its inverse 
    process follows Gaussian distribution according to 
    assumption \ref{Gaussian_noise_assumption}. 
    \begin{align}
        x^{t+1}-{v^{t+1}} = x^{t+1}-D(x^{t+1}, \sigma_t)=x^{t+1}-P(v^{t+1}|v^{t+1}) \sim \mathcal{N}(0, \sigma_t^2)
    \end{align}
    So we get ${v^{t+1}}-x^{t+1}=D(x^{t+1}, \sigma_t)-x^{t+1} \sim -\mathcal{N}(0, \sigma_t^2)$. 
    Combining these two we can get  $v^{t+1}-v^t=b(v^t, y) - \mathcal{N}(0, \sigma_t^2)$, 
    which is a discrete difference equation. 

    Taking an infinitesimal time interval, this will transform to a continuous equation 
    that contains a stochastic term. $dV=b(V, y)dt + \sigma^2 dW_t$, where the $W_t$ is 
    a standard Gaussian process (Wiener process). 
    So this resembles a Ito type Stochastic Differential Equation: 
    \begin{align}
        dX_t = b(t,X_t)dt + \sigma(t,X_t)dW_t
    \end{align}
    where $b_i(t,x), \sigma_{ij}(t,x)$ are the drift term and diffusion term in SDE, 
    respectively. By the definition of SDE, these two function should both be 
    Borel-measurable, and in this case 
    $b(\cdot)$ take an extra input $y$ and $\sigma_{ij}(t,x)$ only depends on $t$. 

\end{proof}

\subsection{Another Proof Based on Backward Markov Process}\label{Backward Markov Process}

\begin{proof}[Proof of Theorem \ref{main_theorem}]
    The Markov chain description of PnP iteration is natural 
    because the trajectory of intermediate variables are random variables and 
    this is a discrete process. 

    In a simplified PnP framework without auxillary variable $u=u+(x-v)$, just the 
    two step iterations: 
    \begin{align}
        {x^{t+1}} &= h(v^t, y)  \\
        {v^{t+1}} &= D(x^{t+1}, \sigma_t) 
    \end{align}
    Combining these two steps, we get the transformation function: 
    \begin{equation}
        {v^{t+1}} = D( h(v^t, y) , \sigma_t)  
    \end{equation}
    And this transformation function decides the kernel of the corresponding 
    Markov chain $P(v^{t+1}|v^{t})$. 

    Then, the Markov Chain of can be transformed into a continuous Markov process 
    description by taking an infinitesimal time interval. 
    And continuous Markov process can be described by the Fokker-Plank equation, 
    which is an differential SDE. 
    \begin{equation}
        \frac{\partial v}{\partial t} = \nabla_v P(v^{t+1}|v^{t})={v^{t+1}} -v^t, 
    \end{equation}
    The corresponding intergration SDE is that: 
    \begin{align}
        v_{t} &= v_0 +  \int_{0}^{t}  \nabla_v P(v^{t+1}|v^{t}) dt   \\
        &=v_0 +  \int_{0}^{t}[x^{t+1}-v^t]dt + [v^{t+1}-x^{t+1}] dt,  \\
        &=v_0 +  \int_{0}^{t}b(t,v_t) dt - \sigma(t) dW_t,  
    \end{align}
    where $dW_t$ is the standard Gaussian process (Wiener process), $b(t,v_t), \sigma(t)$ are 
    the drift term and diffusion term in SDE, respectively. The third equivalence 
    depends on the formulation of drift function $b(\cdot)$ and 
    the Gaussian denoising assumption \ref{Gaussian_noise_assumption}.

\end{proof}

\section{Convergence Analysis}

With the SDE description of PnP, we can transform the convergence analysis of PnP into an 
equivalent SDE solvability problem. 
There are two kinds of solvability of SDE \cite[Chapter~5]{1998GTM113_SDE}, strong 
and weak solutions. 
% We give a brief introduction to these two kinds of solutions, and their 
% existence and uniqueness condition respectively. 
We form a unified framework for PnP convergence property by these two solvability, 
including the traditional fixed-point or non-expansive ones as strong solvability and 
our novel weak convergence conditions according to weak solvability. 
\begin{proposition}[Bridge between PnP convergence and SDE solvability.]
    If the SDE description of PnP is solvable, then it achieves to a certain solution (strong 
    solvability) or obey the probabilistic law of solutions (weak solvability). 
    Then, the corresponding PnP algorithm converges to that 
    stationary point (strong convergence) or the stationary distribution (weak convergence). 
\end{proposition}

However, we will show later on that it is non-trival to apply 
the SDE solvability conditions to PnP convergence analysis. In a word, the strong convergence 
condition about PnP mapping function $h(\cdot)$ would be different with the standard 
Lipschitz condition by a constant multiplier, and the weak convergence condition about 
PnP mapping function $h(\cdot)$ would be a relaxed Lipschitz condition, instead of bounded.

\subsection{Preliminary: Strong and Weak Solvability of Ito-type SDE}

Given a typical Ito-type stochastic differential equation: 
\begin{equation}
    dX_t = b(t,X_t)dt + \sigma(t,X_t)dW_t, 
\end{equation}
where $b_i(t,x), \sigma_{ij}(t,x)$ are Borel-measurable functions. 
And $W_t$ is a standard Gaussian process(Wiener process), which is the main reason why 
SDE suffers from the problem of solvability that is different to ODE. 

Briefly summarizing, the strong solvability requires a pathwise consistency (i.e. a determined 
evolution path and only one terminal point as the solution), and the 
strong solvability condition needs $b_i(t,x), \sigma_{ij}(t,x)$ both to be Lipschitz continuous. 
The weak solvability of SDE is under probability law, so the path would not consistent. 
The weak condition only needs $b_i(t,x), \sigma_{ij}(t,x)$ both to be bounded. 
For more details we would like to refer to \cite[Chapter~5]{1998GTM113_SDE}, and we 
also exhibit them in the appendix. 
% \ref{Strong solution definition}\ref{existence and Uniqueness of Strong Solution}
% \ref{Weak Solution definition}\ref{existence and Uniqueness of Weak Solutions}.

\subsection{Strong Convergence and Lipschitz Condition}

% Similarly to the definition of strong solution in SDE \cite[Definition~5.2.1]{1998GTM113_SDE}, 
First, we will define strong convergence for the PnP algorithm. 
\begin{definition}[Strong convergence]\label{strong_convergence}
    Given PnP iteration
    \begin{equation}
        x^{t+1}=D(h(x^t;y), \sigma_t), 
    \end{equation}
    and initial distribution $x^0\sim\xi$,
    strong convergence claims that the intermediate variable sequence 
    $\{x^0,x^1,\cdots,x^T\}$ 
    is Cauchy sequence. 
\end{definition}

% So $x^T$ approximates 
% a finite constant within the neighbor of it, if $\forall t>T$,
% \begin{equation}
%     x^t<c\pm\epsilon. 
% \end{equation}

Then, we derive the conditions for strong convergence of PnP algorithms as: 
\begin{theorem}[Strong Convergence Conditions]
    Consider a PnP algorithm with iteration: 
    \begin{align}
        x^{t+1}=D(h(x^t;y), \sigma_t). 
    \end{align}
    If both the measurement model $h$ and denoiser $D$ satisfies 
    Lipschitz continuous condition, 
    \begin{align}
        |h(x_{t_1};y)-h(x_{t_2};y)|  &\leq (K+1)|x_{t_1}-x_{t_2}|,  \\
        |D(x_{t_1}, \sigma_{t_1})- D(x_{t_2}, \sigma_{t_2})| &\leq a|x_{t_1}-x_{t_2}|, 
    \end{align}
    Then, strong convergence holds for this PnP iteration. 
\end{theorem}

\begin{proof}%[Strong Convergence Conditions]
    The discrete iteration of PnP algorithm 
    \begin{equation}
        x^{t+1}=D (h(x^t;y), \sigma_t)
    \end{equation} 
    can be 
    described by a continuous SDE as stated in \ref{main_theorem}: 
    \begin{align}
        x^{t+1}-v^t =b(x^t, y) + \sigma_t^2dW_t 
    \end{align}
    Where $b(x, y)=h(x,y)-x$ is the drift term of the PnP-SDE, the diffusion term $\sigma_t$ 
    only depends on time $t$, which is exactly the variance parameter $\sigma_t$ 
    of the Gaussian denoiser $D(\cdot, \sigma_t)$ of PnP. 
    
    Substituting $h$ by the drift term $b(t,x)$ in SDE gives that 
    \begin{equation}
        |b(t,x)-b(t,y)|\leq K|x-y|. 
    \end{equation}
    With the inverse interpretation of the pure Gaussian process for denoiser $D$, 
    we can constrain the diffusion term by the variance of this Gaussian process: 
    \begin{equation}
        |\sigma(t,x)-\sigma(t,y)|\leq a|x-y|. 
    \end{equation}
    Then, according to Yamada and Watanabe's Theorem 
    \cite[Proposition~5.2.31]{1998GTM113_SDE}, the corresponding SDE 
    \begin{equation}
        dX_t = b(t,X_t)dt + \sigma(t,X_t)dW_t 
    \end{equation}
    maintains strong uniqueness for its strong solution. So that strong 
    convergence holds for the given PnP iteration. 
\end{proof}

% The non-expansive confinement is hardly satisfied by practical denoisers, a 
% counterexample would be the Non-Local Mean algorithm, let alone these black-box deep 
% denoisers. 

Above Lipschitz conditions with an extra scaling condition $a<1$ equal to 
those in \cite{2016TCI_pnp_nonexpansive}, and this scaling condition can be 
achieved by restricting PnP mixing weight $\sigma$ in \equref{eq:denoiser}. 

There is a relaxed Lipschitz condition in \cite{2016TCI_pnp_bounded}: 
\begin{equation}
    |D (x_{t_1}, \sigma_{t_1})-x_{t_1}| \leq \sigma C
\end{equation}
where $C$ is a large enough constant. A simple deduction goes from 
Lipschitz's condition would give:
\begin{align}
    |D (x_{t_1}, \sigma_{t_1})- D (x_{t_2}, \sigma_{t_2})| &\leq a|x_{t_1}-x_{t_2}|  \\
    \Rightarrow \qquad  |D (x_{t_1}, \sigma_{t_1})-x_{t_1} | &\leq \frac{(a-1)}{2}|x_{t_1}|      
\end{align}
So this relaxed Lipschitz condition for the denoiser remains to be a variant of Lipschitz  
condition with the choice $x_{t_2}=0$ and an upper-bound for the variable 
$|x_t| \leq \frac{2\sigma C}{a-1}$. 
% More precisely, this Lipschitz continuous condition should be 
% considered as asymptotically Lipschitz continuous, on account of the remaining 
% relation to variable $x_t$. 

Hence, the strong converngence condition of Lipschitz bounded measurement function and 
denoiser in PnP coincides with most previous theoretic analysis of PnP convergence 
guarantee. And the deterministic evolution path of SDE strong solution can be viewed as 
a direct generalization of the fixed point analysis of Partial Differential Equations (PDE) 
described dynamic systems, there is no doubt that a similar methodology would be 
adopted within broad fixed-point-based convergence analysis. 

These conditions bring too strong restrains on the denoiser, for example in 
\cite{2019ICML_pnp-Provably} the author propose to normalize the spectral of neural network 
layers to construct the Lipschitz continuous denoiser. On the contrary,   
unlimited denoisers perform quite well experimentally, even in the starting era of 
PnP-ADMM algorithm\cite{2013GlobalSIP_pnp,2017cvpr_KaiZhang_DnCNN}. 
The gap between theory and application will be bridged by the weak convergence condition in 
the following subsection.

\subsection{Weak Convergence and Bounded Denoiser}

% Weaker condition would be more suitable for most denoiser in modern image processing.
First, we show the weak convergence definition according to the weak solution in SDE 
\cite[Definition~5.3.1]{1998GTM113_SDE}: 

\begin{definition}[Weak Convergence]
    Weak convergence for PnP algorithm is a combination $(X, W), (\Omega,\mathscr{F},P)$,  
    where $W$ is an Brownian motion, $(\Omega,\mathscr{F},P)$ is a probability 
    space, and the distribution of $X$ and $W$ obey a sampling trajectory (a filtration) 
    in $(\Omega,\mathscr{F},P)$. If this combination was given, 
    the iteration $\{x^0,x^1,\cdots,x^T\}$ would be finite valued and the final distribution 
    $P(x^T)$ remains the same under probability law.
\end{definition}

Weak convergence of PnP is defined as related to the distribution along the dynamics. Next, 
we want to derive the existing condition of weak convergence. 

% The existing condition of weak solutions can not be applied, here we start 
% from a deduction for weak existence: 
% \begin{theorem}[Stroock \& Varadhan (1969)]\label{StroockVaradhan}
%     Consider the stochastic differential equation 
%     \begin{equation}
%         dX_t = b(X_t)dt + \sigma(X_t)dW_t 
%     \end{equation} 
%     where the coefficients $b_i, \sigma_{ij}:\mathbb{R}^d\rightarrow\mathbb{R}$ are 
%     bounded and continuous functions. Corresponding to every initial distribution $\mu$ 
%     on $\mathscr{R}(\mathbb{R}^d)$ with 
%     \begin{equation}
%         \int\limits_{\mathbb{R}^d} \|x\|^{2m} \mu(dx)<\infty, 
%     \end{equation}
%     for some $m>1$, a weak solution of the SDE exists. 
% \end{theorem}

% the existence and uniqueness condition is only the uniformly bounded drift term, 
% without any constraints concerning the diffuse term. 

\begin{theorem}[Weak convergence conditions]
    Consider the PnP iteration 
    \begin{equation}
        x^{t+1}=D (h(x^t;y), \sigma_t), 
    \end{equation}
    If the measurement model $h$ is relaxed Lipschitz continuous
    \begin{equation}
        |h(x_t)- x_{t}|\leq \kappa,
    \end{equation}
    while denoiser $D$ is only bounded
    \begin{align}
        |D (x_{t_1}, \sigma_{t_1})| \leq \alpha, 
    \end{align}
    Then, weak convergence holds for this PnP iteration. 
\end{theorem}
\begin{proof}%[Weak convergence conditions]
    The discrete iteration of PnP algorithm 
    \begin{equation}
        x^{t+1}=D (h(x^t;y), \sigma_t), 
    \end{equation} 
    can be 
    described by a continuous SDE as stated in \secref{Inversed PnP-ADMM}: 
    \begin{align}
        x^{t+1}-v^t =b(x^t, y) + \sigma_t^2 dW_t 
    \end{align}
    Where $b(x, y)=h(x,y)-x$ is the drift term of the PnP-SDE, the diffusion term $\sigma_t$ 
    only depends on time $t$, which is exactly the variance parameter of the Gaussian 
    denoiser $D(\cdot, \sigma_t)$ of PnP. 

    Substituting the measurement model by the drift term in SDE via the equation 
    $h(x_t)=b(x_t)+ x_{t-1}$ gives that: 
    \begin{equation}
        |b(x_t)|\leq \kappa. 
    \end{equation}
    The diffusion term is the inversion of the denoising process, so that is 
    bounded as well: 
    \begin{equation}
        \sigma \leq \alpha. 
    \end{equation}
    Then we adopt Stroock \& Varadhan Theorem \cite{1969_Stroock_Varadhan_DiffusionPW}, 
    that bounded drift and diffusion terms indicate the existence of a weak solution. 
    Then, the corresponding PnP iteration is weakly converged. 
\end{proof}

Consequently, the measurement model of PnP should be Lipschitz continuous, which is not a 
big deal for most image restoration tasks as stated in 
\cite{2016TCI_pnp_nonexpansive,2016TCI_pnp_bounded}. 
The denoiser only needs to be bounded instead of Lipschitz continuous. 

\begin{remark}[Bounded Denoisers]
    Deep denoisers with a bounded output layer, like Sigmoid and Tanh function, would satisfy 
    this weak condition. 
    Even without an explicit bounding tool in the output module, 
    a well-functioning denoiser would be implicitly bounded, since the desired images 
    are within a target distribution embedded on the manifold of all images 
    \cite{1984_structure_of_images,1994_scaleSpaceTheory}. Outliers that diverge violently 
    would be punished so that the nature of denoising encourages bounded conditions. 
\end{remark}

\subsection{Relation with Convergence Conditions for Markov Chain}

We have viewed the PnP iteration sequence as a Markov chain as above. Therefore, an 
inquiry would be, can we analyze the convergence of PnP by means of convergence condition (detailed balance) of the Markov chain? Of course, the convergence property of discrete 
Markov chain is pretty mature, and even its applications are fruitful. 
For example, Markov Chain Monte Carlo 
algorithm \cite{1953Metropolis_MCMC} is founded on detailed balance, 
so that its resulting samples follow that target distribution. 

Some research in statistical physics \cite{1971Physik,1972Physik} has already 
investigated the relationship between the solution of the Fokker-Planck equation and detailed 
balance. Their resulting convergence condition is bounded eigenvalues for the 
coefficients of the Fokker-Planck equation, which is consistent with the weak solution of 
SDE and weak convergence of PnP as above. 
% Detailed balance is also satisfied in the sense of 
% probability law. 
We can concrete PnP convergence in an analog of detailed balance, that the drift term would contribute either tiny enough or pure Gaussian 
noise which can be handled thoroughly by the denoiser.

\section{Related Works}

% \subsection{Convergence Analysis of PnP ADMM Alogrithm}
% Theoretical analysis on the convergence and other properties of Plug-and-Play prior 
% methods start as soon as the emergence of the algorithm itself. 
In the founding work 
of the PnP algorithm \cite{2013GlobalSIP_pnp}, the authors did not discuss the theoretical 
convergence properties of the Plug-and-Play framework. 
They just mentioned that the 
convergence guarantee of ADMM algorithms would succeed in PnP-ADMM. 
%  that is, these two functions are convex, closed 
%, and proper functions and the saddle point exist. 
% Experimentally, they showed that 
% even with denoisers that do not explicitly correspond to a convex function, the PnP 
% algorithm still produces a stable result. 
Following theoretical research \cite{2016TCI_pnp_nonexpansive} of the PnP algorithm 
presented the sufficient conditions of the denoising operator that strictly 
guarantee the convergence of the PnP algorithm. The non-expansive condition of the 
denoiser makes the denoising operator a proximal mapping. And then, we can obtain an explicit expression of the convex prior function according to the denoiser. 

% Obviously, t
The non-expansive property is too strong to be satisfied by commonly used 
denoisers like NLM and BM3D. %\cite{2016TCI_pnp_bounded} 
\cite{2016TCI_pnp_bounded} gave a clear analysis of the fixed-point 
convergence of the PnP framework by constructing a Cauchy sequence of variables during 
the iteration in a Banach space. This fixed-point convergence relies on the Lipschitz 
continuous property of the likelihood function and the denoiser. However, they 
achieved this only under diminishing stepsizes. Thus, non-expansive condition can easily converted into Lipschitz 
continuous via a scale factor. 

In \cite{2019ICML_pnp-Provably}, the author constructed a Lipschitz-bounded 
deep denoiser to provide provably convergence property for the PnP method. 
PnP-SGD \cite{2022JMIV_PnP_SGD} uses a stochastic gradient descent method to solve the PnP 
problem. Its convergence depends on the contractive residual condition on the denoiser. 
The above theoretical analyses all imposed seemingly strong constraints on the denoiser. 
In the meantime, experimental results \cite{2017cvpr_KaiZhang_DnCNN,2021PAMI_kaizhang_DPIR} 
surprisingly rush ahead of the corresponding theoretical researches. 
Apparently, there is a huge gap between theoretical research and practical experiments. 
% A possibly burgeoning breakthrough, in our opinion, may be related 
% to the score-matching interpretation of the PnP algorithm. 
Our novel interpretation of the PnP algorithm by SDE could possibly bypass the previously 
indispensable conditions on denoisers, with only a mild assumption of bounded outputs 
of the denoising algorithm embedded.

\section{Examples}
The supportive example is an experimental comparison between PnP algorithms with 
Lipschitz continuous denoiser and bounded denoiser. And the counterexample of unbounded 
denoiser is presented by deduction. 

\subsection{Lipschitz Continuous Denoiser Or Bounded Denoiser}
In the baseline work of Lipschitz continuous deep 
denoiser within PnP\cite{2019ICML_pnp-Provably}, the authors give the relationship 
between different $\alpha$ as the raio parameter and 
contraction factor. With $\sigma^2=\alpha \gamma$ and fixed $\gamma$, the 
parameter $\alpha$ is then closely related to the variance parameter $\sigma$. 
So we adopted the same experimental setting up as in \cite{2019ICML_pnp-Provably}, 
we use the relationship between the ratio parameter $\alpha$ and PSNR. 

We compare PnP with the vanila DnCNN and with the Real Spectral Normalized\cite{2019ICML_pnp-Provably} DnCNN as shown in \tabref{table:alpha_Lipschitz}. 
The previous theoretical analysis in \cite{2019ICML_pnp-Provably} considered the Lipschitz continuous condition of DnCNN is the key for PnP-CNN to 
converge. But in the experimental results below we show that although with little performance drop, the PnP-CNN converges 
as well, instead of diverges. 

\begin{table}[h]
    \caption{PSNR values for convergence of vanila denoiser(PnP-CNN) and 
    Lipschitz continuous denoiser(PnP-RSN-CNN) with varying noise variance 
    related ratio parameter $\alpha$.
    }
    \label{table:alpha_Lipschitz}
    \vskip 0.15in
    \begin{center}
    \begin{small}
    \begin{sc}
    \begin{tabular}{lccccccc}
    \toprule
      $\alpha$ & 0.001 & 0.003 & 0.01 & 0.03 & 0.1 & 0.3 & 1.0   \\
    \midrule
    PnP-RSN-CNN   &4.74984      & 5.33115       &6.75710      & 9.11250       &12.95924      & 17.1120       & 19.47662         \\
    PnP-CNN   & 3.37694      & 4.10075       &5.63357      & 7.86317       &12.23303      & 16.9121       & 19.39487      \\
    \bottomrule
    \end{tabular}
    \end{sc}
    \end{small}
    \end{center}
    \vskip -0.1in
\end{table}

We would also present in \figref{pic:twoModel} the evolution dynamics of above two methods 
when $\alpha=1.0$. 
The resulting PSNR curves of these two methods are close 
(even superposition) with each other, and hard to tell between them. 
We can see that the main advantage of Lipschitz continuous denoiser is their 
smoother trajectories. As a conclusion, with/without Lipschitz continuous denoisers, 
the PnP algorithm converges, which indicates that the Lipschitz condition is not 
a necessary guarantee for PnP convergence.  

\definecolor{orange}{rgb}{0.8, 0.4, 0.1}
\definecolor{indigo}{rgb}{0.1, 0.4, 0.6}
\begin{figure}[h]
    \centering
    \includegraphics[width=1.0\columnwidth]{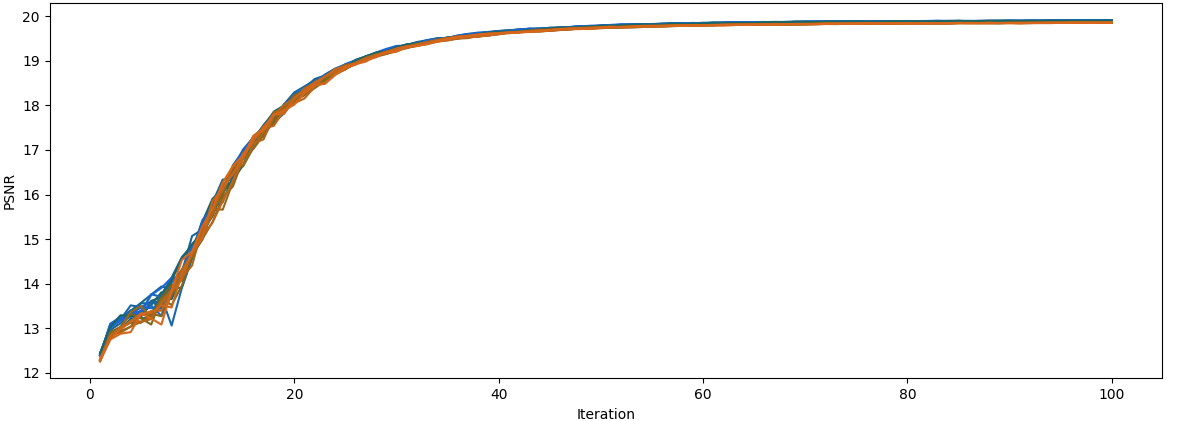}
    \caption{PSNR curves during PnP iteration comparing two methods. 
    \textcolor{orange}{Orange} lines show the results of PnP with Lipschitz 
    continuous denoisers (PnP-RSN-CNN), while \textcolor{indigo}{Indigo} lines are with 
    Lipschitz discontinuous denoisers (PnP-CNN). Lipschitz continuous 
    denoiser ensures a smoother trajectory. 
    In general, both Lipschitz continuous and discontinuous denoiser 
    converges.}
    \label{pic:twoModel}
\end{figure}

\subsection{Unbounded denoiser}

A counterexample of divergence PnP iteration % instead of convergence 
may employ an unbounded denoiser. 
\begin{remark}[PnP Divergence]
    PnP iteration $v^{t+1}=D(h(v^t;y), \sigma_t)$ with an unbounded denoiser $D(\cdot)$ 
    would diverge unless with a contractive measurement function $h(\cdot)$. 
\end{remark}
Usually, the measurement function would not perform as contractive operator, 
thus do no good to the convergence of PnP. Hence, PnP with unbounded denoiser 
would diverge in general. 

An unbounded denoiser would not perform well in denoising tasks, instinctively. 
At least, a well-performed denoiser is bounded by these pixel values' range. 
Even a deep denoiser without the explicit restriction of its output range would follow 
an implicit bounded strategy if it was trained on an elaborate dataset. 
Since the training dataset is bounded beforehand, extreme pixel values are doomed to be 
penalized for the extreme gain of loss. 
\begin{remark}[Bounded Denoiser]
    With empirical risk minimization training of deep denoiser, the expected risk would 
    be the approximating error of training, the exceed risk corresponds to the 
    generalization error. 
    \begin{align}
        \ell_{empirical} = \ell_{expected} + \ell_{exceed}. 
    \end{align}
    With over-parameterized deep neural network and sufficient amount of data for training, 
    the expected risk would approximate zero. 
    For Gaussian distributed noise, the exceed risk would be constrained 
    between the range of $[\mu-3\sigma, \mu+3\sigma]$ 
    with parabability of 99.74\%. Note that we set the sequence $\sigma_t$ to decay  
    as $t$ grows. So the bounded denoiser within range $[-256-3\sigma_T, 256+3\sigma_T]$
    is guaranteed with high probability, where $T$ is the stopping time of PnP. 
\end{remark}

% So we have to stop after the above deduction because unbounded denoiser 
% exceed the range of practical implementations. 

\section{Conclusion}

To close the gap between practical PnP with advanced denoiser prior and theoretical 
convergence analysis, we considered the continuous counterpart of the PnP algorithm, 
and related it to stochastic differential equations. 

Then, with the aid of the solvability theory of SDE, we 
give a unified framework for convergence analysis of the PnP algorithm. 
Strong solution corresponds to Lipschitz's continuous condition in previous research. 
With the weak solution under probability law, 
we proposed weak convergence for PnP analysis. Then, we obtained a much weaker convergence 
condition of only bounded denoisers and Lipschitz continuous measurement models.
Finally, we give 
a few counterexamples to further illustrate our convergence theory for PnP algorithms.

Most current SOTA denoisers are not Lipschitz continuous but bounded inherently, 
provided that the denoising task functions as expected. So weak 
conditions are more suitable for practical denoisers than previous ones, which would 
further potentiate a more confident usage of the PnP algorithm with more advanced 
bounded denoisers.  

% \section{Societal Impact and Limitations}
% This paper presents a theoretical convergence analysis on a specific PnP 
% optimization method. 
% And our proof makes the trival application of non-Lipschitz denoiser to 
% PnP image restoration tasks reasonable. 
% For now, the direct social impact is quite mild. 
% Abd the social impact after the possible applications and implementations 
% for engineering should be re-assessed and taken care of. 

% The limitation we see for now is that it does not provide a useful algorithm for 
% application. 

% Among the major limitations we currently see is probable existence of the bias of the proposed node
% classification approach due to potentially unbalanced population groups [5]. Indeed, in the context of
% social data analysis, given that we analyze peer neighborhoods which themselves tend to be largely
% formed by the same users as the target node and hence are likely to be dominated by the abundant
% group, such bias is very likely to exist and to be non-negligible. Mitigating it is a problem on its
% own but some ideas include using multiple filtration functions based on different node attributes so
% the node neighborhoods are made more diverse and are no longer dominated by a specific group or
% removing some potentially sensitive node attributes (e.g., gender) completely from the study (which
% has its own pros and cons in terms of accuracy).

{
\small

% [1] Alexander, J.A.\ \& Mozer, M.C.\ (1995) Template-based algorithms for
% connectionist rule extraction. In G.\ Tesauro, D.S.\ Touretzky and T.K.\ Leen
% (eds.), {\it Advances in Neural Information Processing Systems 7},
% pp.\ 609--616. Cambridge, MA: MIT Press.

% [2] Bower, J.M.\ \& Beeman, D.\ (1995) {\it The Book of GENESIS: Exploring
%   Realistic Neural Models with the GEneral NEural SImulation System.}  New York:
% TELOS/Springer--Verlag.

% [3] Hasselmo, M.E., Schnell, E.\ \& Barkai, E.\ (1995) Dynamics of learning and
% recall at excitatory recurrent synapses and cholinergic modulation in rat
% hippocampal region CA3. {\it Journal of Neuroscience} {\bf 15}(7):5249-5262.

%%%%%%%%%%%%%%%%%%%%%%%%%%%%%%%%%%%%%%%%%%%%%%%%%%%%%%%%%%%%
\bibliography{main}
\bibliographystyle{plain}
}

\appendix

\section{Related Works}
Other than Plug-and-Play ADMM framework, there are also regularization by denoiser (RED) 
and consensus equilibrium (CE) frameworks as the variants of PnP. 
We also review them and beyond for the completeness of the related works.

\subsection{Convergence Analysis of RED}
Instead of proximal mapping theory in PnP-ADMM, the regularization by denoiser \cite{2017SIAM_RED} 
framework 
offered a systematic use of denoisers for regularization in inverse problems. 
The regularization term is  
\begin{equation}
    \rho(x) = \frac{1}{2} x^T[x-f(x)], 
\end{equation}
where the denoiser takes the candidate's image as input. 
Intuitively, the penalty term is proportional to the inner product between the image and its denoising residual. 
With the aid of the RED framework, the authors proposed the following necessary 
conditions on the denoiser $f(x)$ to guarantee convergence: 
\begin{itemize}
    \item Local Homogeneity: $\forall x\in \mathbb{R}^n, f((1+\epsilon)x)=(1+\epsilon)f(x)$ for sufficiently small $\epsilon >0$.
    \item Differentiability: The denoiser $f(\cdot)$ is differentiable where $\nabla f$ denotes its Jacobian. 
    \item Jacobian Symmetry: $\nabla f(x)^T=\nabla f(x), \forall x\in\mathbb{R}^n$. 
    \item Strong Passivity: The spectral radius the Jacobian satisfies $\eta(\nabla f(x))\leq 1$. 
\end{itemize}
% Consequently, the regularization term is differentiable, and convex, and its gradient is given 
% by the denoising residual $x-f(x)$, so that the global convergence is guarantee regardless 
% of which optimization method has been used. 

There are two highlighting researches about improving RED. %One 
\cite{2019TCI_scoreMatching_RED_PnP}
interpreted the denoiser using score-matching 
theory, and bypassed the constraint of Jacobian symmetry. 
% The other \cite{2021SIAM_RED-PRO}
\cite{2021SIAM_RED-PRO} introduced the assumption of demi-contractive denoisers, 
which is a relaxation compared to the previous convergence analysis under the fixed-point 
theory. 

\subsection{Other Convergence Analysis of PnP}
Another mainstream framework is consensus equilibrium. 
% It was 
% built directly on the idea of fixed-point iteration, which is more essential than the PnP 
% framework. Then, we can utilize 
In \cite{2018SIAM_consensus_equilibrium}, the authors used the convergence of Mann 
iteration during consensus 
equilibria solving to guarantee the convergence of the corresponding PnP algorithm. 

Deep unfolding methods \cite{2014_deepUnfolding} also constitute a particular 
viewpoint of the PnP algorithm. 
% The authors of \cite{2019PAMI_lyapunov} 
\cite{2019PAMI_lyapunov} centered around the denoising deep neural network 
and treated the observation model as a priori. Then, they proposed a Lyapunov function-based 
convergence guarantee. This Lyapunov condition, again, assured that the denoiser should 
be Lipschitz bounded. 

The convergence of other variable splitting algorithms, like 
ISTA \cite{2016SIAM_ISTA_FISTA_convergence,2018NIPS_ISTA_convergence}, 
AMP \cite{2014_AMP_convergence} etc., are also 
analyzed via fixed-point mathematic tools.

% \cite{2022JMLR_tfpnp} is about the auto-tuning of hyper-parameters by using reinforement learning, 
% which is a milestone in the development of PnP theory.  
% \cite{2023MSP_pnp_overview} is a thorough overview of PnP methods. 
% \cite{2023TOG_delta_prox} proposed a differentiable proximal operator for efficient optimization modeling.

% \subsection{Brief Summary}
% The above theoretical analyses all imposed seemingly strong constraints on the denoiser. 
% In the meantime, experimental results \cite{2017cvpr_KaiZhang_DnCNN,2021PAMI_kaizhang_DPIR} 
% surprisingly rush ahead of the corresponding theoretical researches. 
% Apparently, there is a huge gap between theoretical research and practical experiments. 
% % A possibly burgeoning breakthrough, in our opinion, may be related 
% % to the score-matching interpretation of the PnP algorithm. 
% Our novel interpretation of the PnP algorithm by SDE could possibly bypass the previously 
% indispensable conditions on denoisers, with only a mild assumption of bounded outputs 
% of the denoising algorithm embedded. 

\section{Experiments}

% We want to state that the main result of our work is the convergence condition given 
% in the former chapter.
According to previous theoretical convergence analysis\cite{2016TCI_pnp_nonexpansive}, 
the vanila CNN is expansive and without Lipschitz continuous condition, 
so vanila PnP-CNN should diverge. But in fact the vanila PnP-CNN 
converges in so many previous works and performs pretty well in downstream 
application of image restoration. 
% Hence, we do not intend to repeat the well-established pipeline of PnP methods. 
% There are just a few trival experiments to help the illustration 
% of our theoretical results. And 
We added extra experiments to investigate the performance 
% of Stong/Weak 
% convergence and 
With/Without Lipschitz continuous condition over different noise variance.
% in \secref{Extra Experiments}. 

\subsection{PnP-SDE Implementation}

We would call our proposed algorithm PnP-SDE. To keep consistency with the SDE description, 
we need to add noise during iteration in order to converge 
in the sense of probability law instead of pathwise. Because the denoising process of a 
give trained denoiser is deterministic, without stochastic trajectory during denoising. 
% For more detailed reasons, please refer to 
% \secref{Relationship with Score-Denoiser}.
We should elucidate the denoiser via its 
variance $\sigma$. 

Normally, PnP-ADMM algorithms would use a scale factor to control the dynamics of 
PnP iteration \cite{2013GlobalSIP_pnp,2016TCI_pnp_bounded,2016TCI_pnp_nonexpansive,2021NIPS_diffusion_pnp}. 
We can correspond the scale factors to the mixing weight and the Lagrangian multiplier. 
Intuitively, this progressive evolution would stratify the processing of 
different granular features so that each feature's granules would be less degraded by 
more intense noise levels. 

However, from the viewpoint of the philosophy of fundamentalism, \cite{2019ICML_pnp-Provably}
insists that we should analyze the PnP convergence property without these scale factors. To some extent, this research is more like an 
ablation variant of the PnP algorithm, which violates the usual setting for PnP. 
Our formulation of the SDE description for PnP follows the usual setting, which also intends 
to satisfy the diffusion analog of the inverse of denoising. 

For the demonstration of counterexample, Lipschitz discontinuous denoiser, we would 
adopt the same fundamentalism-style setting that discards the confounding of scaling 
factors. 
More 
detailed experiments are as follows.

% The denoisers without Lipschitz continuous property, but bounded, should be the most 
% proper ones to be used for demonstrating convergence in the sense of probability law, 
% instead of pathwise designed. 

% add different noise, and converge into the same distribution. 
% Convergence Curve MSE 
\subsection{Weak Convergence}

Weak convergence with random noise added during PnP iteration is similar to 
the implementation in \cite{2021NIPS_diffusion_pnp}. However, a distinctive difference is 
that our PnP-SDE models the noise within each iteration using a precise variance 
parameter $\sigma$, while \cite{2021NIPS_diffusion_pnp} employed bias-free 
denoiser \cite{2020ICLR_biasFree} to achieve a generalization over different noise levels implicitly. So we implement our proposed PnP-SDE using DnCNN denoiser \cite{2017cvpr_KaiZhang_DnCNN,2021PAMI_kaizhang_DPIR} with explicit 
variance $\sigma$ as an extra channel of inputs. 

\definecolor{purple}{rgb}{0.5,0,0.5}
\begin{figure}[th]
    \centering
    \includegraphics[width=.5\columnwidth]{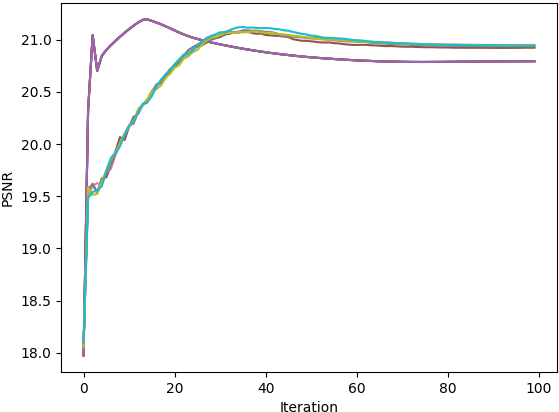}
    \caption{PSNR curves during PnP iteration. \textcolor{purple}{Purple} lines 
    are under strong convergence, and they clearly demonstrate pathwise consistency. 
    These lines with multiple colors are under weak convergence so that they go 
    through different trajectories. 
    Weak convergence also converges at the end of each trajectory, 
    and achieve a bonus on performance.}
    \label{pic:weakConverge_psnr}
\end{figure}
\begin{figure}[th]
    \centering
    \includegraphics[width=.5\columnwidth]{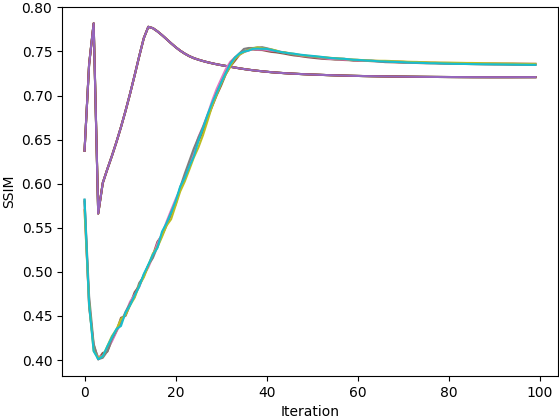}
    \caption{SSIM curves during PnP iteration. The instability at the beginning is due 
    to the independence of SSIM value and noise level. 
    \textcolor{purple}{Purple} lines 
    are under strong convergence and are pathwise unique. 
    These lines with multiple colors are under weak convergence, with 
    different trajectories. 
    Weak convergence also gains a bonus on performance.}
    \label{pic:weakConverge_ssim}
\end{figure}
% \begin{figure}
%     \centering
%     \subfigure{
%       \includegraphics[width=.47\columnwidth]{pic/weakConverge_psnr.png}
%     }
%     % \hspace{0.1in}
%     \subfigure{
%       \includegraphics[width=.47\columnwidth]{pic/weakConverge_ssim.png}
%     }
%     \caption{The original image and simulated echo signal.}
%     \label{pepper}
% \end{figure}

\begin{table}[h]
    \caption{PSNR and SSIM results for strong and weak convergence. The results shows 
    that both of them converge, although with a little performace gap. }
    \label{table:strong_weak}
    \vskip 0.15in
    \begin{center}
    \begin{small}
    \begin{sc}
    \begin{tabular}{lcccr}
    \toprule
        & PSNR & SSIM  \\
    \midrule
    Strong   & 20.5232& 0.7207 \\
    Weak   & 20.6417& 0.7361\\
    \bottomrule
    \end{tabular}
    \end{sc}
    \end{small}
    \end{center}
    \vskip -0.1in
\end{table}

In \figref{pic:weakConverge_psnr} and \figref{pic:weakConverge_ssim}, we can see that 
PnP converges both under strong and weak convergence situations. While strong 
convergence is pathwise unique, weak convergence is in the sense of probability law, 
with different trajectories. Surprisingly, as analyzed before, noises added in 
weak convergence help the PnP algorithm to escape from local minima. Therefore, this achieves extra performance gains, as shown in \tabref{table:strong_weak}. 

\subsection{Lipschitz discontinuous denoiser}

A simple question would be, is there a counterexample that transgresses 
previously proposed conditions on denoisers, but still converges. 

Firstly, we can unify previously proposed constraints on the denoisers 
using Lipschitz bounded \cite{2019ICML_pnp-Provably}. 
Other equivalent expressions like non-expansive \cite{2016TCI_pnp_nonexpansive}, 
contractive property, demi-contractive \cite{2021SIAM_RED-PRO}, can be view as the 
a scaled version of it. 

% First, following \cite{2016TCI_pnp_bounded}, we would affirm here that NLM denoiser is 
% not only expansive, but they are also Lipschitz discontinuous. 
% Nonexpansive: 
% \begin{equation}
%     \kappa = \|\tilde{W}_x(x)-\tilde{W}_y(y)\|^2 /\|x-y\|^2 >1
% \end{equation}
% There would be unbounded gradients because 

% Suppose a huge enough variance $\sigma$, noise $\varepsilon_{\sigma}$, and the 
% contaminated image $y=x+\varepsilon_{\sigma}$. 
% Lipschitz discontinuous: 
\cite{2019ICML_pnp-Provably} proposed to constrain the denoiser's Lipschitz constant by spectral normalization and 1-Lipschitz non-linearity. 
% So the seemingly most 
% complex denoiser would be the easiest one to analyze, for the black-box mapping 
% of its hierarchical modules. For example, the gradients related to the variables refer 
% directly to the weights, which is apparently unbounded. Another kind of module that 
% should be concerned with is the non-linear units, although the commonly used ReLU function is 
% 1-Lipschitz intrisically. 
Even in their experiments, we can see that the PnP algorithm without Lipschitz continuous 
denoisers also converges to the desired image distribution. As they 
demonstrated, there is only a tiny difference between the PSNR of Lipschitz continuous and 
discontinuous ones. 
% \cite{2019ICML_pnp-Provably} has released their 
% code and pre-trained model, which can be quite easily reproduced. So we would rather not 
% repeat the related experiments in this paper. 

% Surprisingly, NLM denoiser and BM3D denoiser are naively Lipschitz discontinuous 
% because of the utilization of similarity between image patches 
% \cite{2011CVPR_mse_denoising}. Briefly speaking, 
% the variance can be upper-bounded by the mean square error between 
% expectation of noise image and unknown clean image, and lower-bound is the sample 
% variance. 

\definecolor{orange}{rgb}{0.8, 0.4, 0.1}
\definecolor{indigo}{rgb}{0.1, 0.4, 0.6}
\begin{figure}[h]
    \centering
    \includegraphics[width=.5\columnwidth]{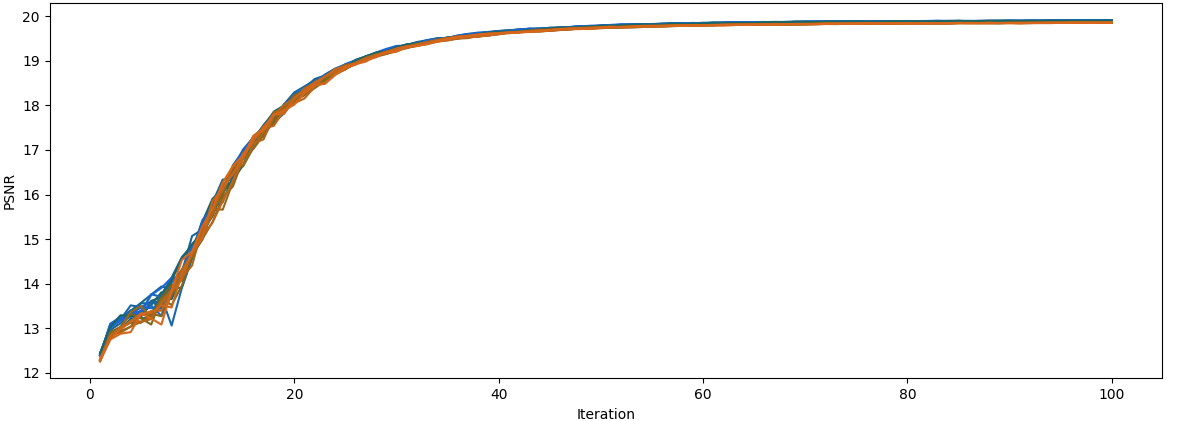}
    \caption{Performance curves during PnP iteration under weak convergence. 
    \textcolor{orange}{Orange} lines show the results of PnP with Lipschitz 
    continuous denoisers, while \textcolor{indigo}{indigo} lines are with Lipschitz 
    discontinuous denoisers. Lipschitz continuous 
    denoiser ensures a slightly smoother trajectory and also shows a tiny performance 
    gain. In general, both Lipschitz continuous and discontinuous denoiser converges 
    under weak probability uniqueness.}
    \label{pic:twoModel}
\end{figure}

In fact, practical PnP algorithms up-to-date with deep denoisers have already given fruitful expositions without divergence problems. Here, we show the performance 
curve under weak convergence in \figref{pic:twoModel}. We can see that Lipschitz's 
continuous denoiser makes their corresponding trajectories smoother. And this is consistent with what we anticipated. A tiny performance gain is also reasonable for the 
robustness of the Lipschitz continuous denoiser. However, this is a bonus, not the main dish. Apparently, both Lipschitz continuous and discontinuous denoisers converge to the terminal without oscillation. Note that the setting in 
\cite{2019ICML_pnp-Provably} discards the scale factor, which we also adopt here so that we would ignore the disturbance of scaling.

\subsection{Unbounded denoiser}

For the existence of SDE's weak solution, we need a mild condition of bounded drift and 
diffusion term. Therefore, the convergence of PnP-SDE iteration can be guaranteed by 
Lipschitz continuous drift term and bounded diffusion term. 
This condition is so loose that a counterexample %of divergence instead of convergence 
should employ an 
unbounded denoiser. Apparently, it would diverge under this circumstance. 
An unbounded denoiser would not perform well in denoising tasks, instinctively. 

At least, a well-performed denoiser is bounded by these pixel values' range. 
Even a deep denoiser without the explicit restriction of its output range would follow 
an implicit bounded strategy if it was trained on an elaborate dataset. 
Since the training dataset is bounded beforehand, extreme pixel values are doomed to be 
penalized. 

So we have to stop after the above deduction because unbounded denoiser 
exceed the range of practical implementations. 

\section{SDE Strong Solvability}\label{Strong solution definition}
Given a typical stochastic differential equation 
\begin{equation}
    dX_t = b(t,X_t)dt + \sigma(t,X_t)dW_t, 
\end{equation}
where $b_i(t,x), \sigma_{ij}(t,x)$ are Borel-measurable functions. 
The definition of a strong solution for SDE is as follows:

\begin{definition}[Strong solution]
    Given a probability space $(\Omega,\mathscr{F},P)$, the fixed Brownian motion $W$ 
    and initial condition $\xi$, a strong solution of the stochastic differential 
    equation is a process $X=\{X_t;0\leq t<\infty\}$ with continuous sample paths and 
    with the following properties: 
    \begin{itemize}
        \item X is adapted to the filtration $\{\mathscr{F}_t\}$,
        \item $P[X_0=\xi]=1$,
        \item $P[\int_0^t \{|b_i(s,X_s)|+\sigma_{ij}^2 (s,X_s)\}ds<\infty]=1$ holds for every $1\leq i\leq d, 1\leq j\leq r$ and $0\leq t<\infty$, 
        \item the integral 
        \begin{equation}
            X_t = X_0 + \int_0^t b(s,X_s)ds + \int_0^t \sigma(s,X_s)dW_s,
        \end{equation}
        holds almost surely, for $0\leq t<\infty$. 
    \end{itemize}
\end{definition}

% \section{existence and Uniqueness of Strong Solution}\label{existence and Uniqueness of Strong Solution}
Next, we would investigate the existence and uniqueness of strong solutions for a given 
SDE. 

\begin{theorem}[Strong Pathwise Uniqueness]
    Suppose that the coefficients $b(t,x), \sigma(t,x)$ are locally Lipschitz-continuous 
    in the space variable; i.e., for every integer $n\leq 1$ there exists a constant 
    $K_n>0$ such that for every $t\geq 0, \|x\|\leq n$ and $\|y\|\leq n$: 
    \begin{equation}
        \|b(t,x)-b(t,y)\| + \|\sigma(t,x)-\sigma(t,y)\| \leq K_n \|x-y\|. 
    \end{equation}
    Then strong uniqueness holds. 
\end{theorem}

\begin{theorem}[Strong existence]
    Suppose that the coefficients $b(t,x),\sigma(t,x)$ satisfy the global Lipschitz and 
    linear growth conditions 
    \begin{align}
        &\|b(t,x)-b(t,y)\| + \|\sigma(t,x)-\sigma(t,y)\| \leq K \|x-y\|, \\
        &\|b(t,x)\|^2 + \|\sigma(t,x)\|^2 \leq K^2 (1+\|x\|^2). 
    \end{align}
    for every $0\leq t<\infty, x\in \mathbb{R}^d, y\in\mathbb{R}^d$, where $K$ is a positive 
    constant. On some probability space $(\Omega, \mathscr{F}, P)$, let $\xi$ be an 
    $\mathbb{R}^d$-valued random vector, independent of the $r$-dimensional Brownian motion 
    $W=\{W_t, \mathbb{F}_t^W, 0\leq t<\infty\}$, and with finite second moment: 
    \begin{equation}
        E\|\xi\|^2<\infty.
    \end{equation}
    Let $\{\mathscr{F}_t\}$ be the filtration. Then there exists a continuous, adapted process 
    $X=\{X_t, \mathscr{F}_t, 0\leq t<\infty\}$ which is a strong solution of the stochastic 
    differential equation relative to $W$, with initial condition $\xi$. 
    
    Moreover, this process is square-integrable: for every $T>0$, there exists a constant 
    $C$, depending only on $K$ and $T$, such that 
    \begin{equation}
        E\|X_t\|^2 \leq C(1+E\|\xi\|^2)e^{Ct}, 0\leq t<T. 
    \end{equation}
\end{theorem}

There is a further proposition on the uniqueness that disjoins the condition of drift 
term and diffusion term. 

\begin{proposition}[Yamada \& Watanabe]
    Let us suppose that the coefficients of the one-dimensional equation $(d=r=1)$
    \begin{equation}
        dX_t = b(t,X_t)dt + \sigma(t,X_t)dW_t 
    \end{equation}
    satisfy the conditions 
    \begin{align}
        |b(t,x)-b(t,y)|\leq K|x-y|  \\
        |\sigma(t,x)-\sigma(t,y)|\leq a|x-y|
    \end{align}
    for every $0\leq t<\infty$ and $x\in \mathbb{R}, y\in \mathbb{R}$, where $K$ is a 
    positive constant and $a:[0,\infty)\rightarrow [0,\infty)$ is a strictly increasing 
    function with $a(0)=0$ and 
    \begin{equation}
        \int_{(0,\varepsilon)} a^{-2}(u)du=\infty, \forall \varepsilon>0. 
    \end{equation}
    Then strong uniqueness holds for the stochastic differential equation. 
\end{proposition}

\section{SDE Weak Solvability}\label{Weak Solution definition}

\begin{definition}[Weak solution]
    A weak solution of a given stochastic differential equation is a triple $(X, W), 
    (\Omega,\mathscr{F},P), \{\mathscr{F}_t\}$, where 
    \begin{itemize}
        \item $(\Omega,\mathscr{F},P)$ is a probability space, and $\{\mathscr{F}_t\}$ is a filtration of sub-$\sigma$-fields of $\mathscr{F}$ satisfying the usual conditions, 
        \item $X=\{X_t, \mathscr{F}_t; 0\leq t<\infty\}$ is a continuous, adapted $\mathbb{R}^d$-valued process, 
        $W=\{W_t, \mathscr{F}_t; 0\leq t<\infty\}$ is an $r$-dimensional Brownian motion, 
        \item $P[\int_0^t \{|b_i(s,X_s)|+\sigma_{ij}^2 (s,X_s)\}ds<\infty]=1$ holds for every $1\leq i\leq d, 1\leq j\leq r$ and $0\leq t<\infty$, 
        \item the integral 
        \begin{equation}
            X_t = X_0 + \int_0^t b(s,X_s)ds + \int_0^t \sigma(s,X_s)dW_s,
        \end{equation}
        holds almost surely, for $0\leq t<\infty$. 
    \end{itemize}
\end{definition}

A weak solution of SDE is related to the distribution along the dynamics. We 
would remark here that this triple $(X,W), (\Omega,\mathscr{F},P), \{\mathscr{F}_t\}$ 
have to be considered as the solution itself, not assumptions. 

% \section{existence and Uniqueness of Weak Solutions}\label{existence and Uniqueness of Weak Solutions}

\begin{definition}[Weak Probability Uniqueness]
    We say that  uniqueness in the sense of probability law holds for a given stochastic 
    equation if, for any two weak solutions $(X,W), (\Omega,\mathscr{F},P), \{\mathscr{F}_t\}$, 
    and $(\tilde{X},\tilde{W}), (\tilde{\Omega},\tilde{\mathscr{F}},\tilde{P}), \{\tilde{\mathscr{F}_t}\}$, with the same initial distribution 
    $\Gamma$, i.e., 
    \begin{equation}
        P[X_0\in \Gamma] = \tilde{P}[\tilde{X_0}\in \Gamma], \forall \Gamma \in \mathscr{R}(\mathbb{R}^d), 
    \end{equation}
    the two processes $X, \tilde{X}$ have the same law. 
\end{definition}

\begin{proposition}[Weak existence]
    Consider the stochastic differential equation 
    \begin{equation}
        dX_t = b(t,X_t)dt + dW_t, 0\leq t\leq T,  
    \end{equation}
    where $T$ is a fixed positive number, $W$ is a $d$-dimensional Brownian motion, and 
    $b(t,x)$ is a Borel-measurable, $\mathbb{R}^d$-valued function on $[0,T]\times \mathbb{R}^d$ 
    which satisfies 
    \begin{equation}
        \|b(t,x)\|\leq K(1+\|x\|), 0\leq t\leq T, x\in \mathbb{R}^d
    \end{equation}
    for some positive constant $K$. For any probability measure $\mu$ on $(\mathbb{R}^d, \mathscr{R}(\mathbb{R}^d))$, 
    the stochastic differential equation has a weak solution with initial distribution $\mu$. 
\end{proposition}

\begin{theorem}[Girsanov Theorem]\label{Girsanov}
    Given a probability space $(\Omega, \mathscr{F}, P)$ and a $d$-dimensional Brownian 
    motion $W=\{W_t=(W_t^{(1)},\cdots,W_t^{(d)}), \mathscr{F}_t, 0\leq t<\infty\}$ defined 
    on it, with $P[W_0=0]=1$. Let $\mathscr{F}_t$, $X=\{X_t=(X_t^{(1)},\cdots,X_t^{(d)}), 
    \mathscr{F}_t, 0\leq t<\infty\}$ be a vector of measurable, adapted square-integrable 
    processes. We set 
    \begin{equation}
        Z_t(X)\triangleq \exp [\sum_{i=1}^{d} \int_{0}^{t} X_s^{(i)}dW_s^{(i)} -\frac{1}{2}\int_{0}^{t} \|X_s\|^2 ds].
    \end{equation}
    Assume that $Z(X)$ is a martingale. Define a process $\tilde{W}=\{\tilde{W}_t=(\tilde{W}_t^{(1)},\cdots,\tilde{W}_t^{(d)}), 
    \mathscr{F}_t, 0\leq t<\infty\}$ by 
    \begin{equation}
        \tilde{W}_t^{(i)}\triangleq W_t^{(i)} -\int_{0}^{t} X_s^(i)ds, 1\leq i\leq d, 0\leq t<\infty. 
    \end{equation}
    For each fixed $T\in[0,\infty)$, the process $\{\tilde{W}_t, \mathscr{F}_t, 0\leq t\leq T\}$ is a $d$-dimensional 
    Brownian motion on $(\Omega, \mathscr{F}_T, \tilde{P}_T)$. 
\end{theorem}

\begin{proof}
    We begin with a $d$-dimensional Brownian family $X=\{X_t, \mathscr{F}_t, 0\leq t\leq T\}, 
    (\omega,\mathscr{F}), \{P^x\}_{x\in \mathbb{R}^d}$. Let 
    \begin{equation}
        Z_t(X)\triangleq \exp [\sum_{i=1}^{d} \int_{0}^{t} X_s^{(i)}dW_s^{(i)} -\frac{1}{2}\int_{0}^{t} \|X_s\|^2 ds].
    \end{equation}
    be a martingale under each measure $P^x$, so the Theorem \ref{Girsanov} implies that, 
    under $Q^x$ given by $(dQ^x/dP^x)=Z_T$, the process 
    \begin{equation}
        W_t \triangleq X_t - X_0 - \int_{0}^{t} b(s,X_s)ds, 0\leq t\leq T 
    \end{equation}
    is a Brownian with $Q^{\mu}(A)\triangleq \int_{\mathbb{R}^d} Q^x(A)\mu(dx)$, the triple 
    $(X,W), (\Omega,\mathscr{F},P), \{\mathscr{F}_t\}$ is a weak solution of the SDE. 
\end{proof}

% \section{Proof of Stroock \& Varadhan Theorem}

\begin{theorem}[Stroock \& Varadhan (1969)]\label{StroockVaradhan}
    Consider the stochastic differential equation 
    \begin{equation}
        dX_t = b(X_t)dt + \sigma(X_t)dW_t 
    \end{equation} 
    where the coefficients $b_i, \sigma_{ij}:\mathbb{R}^d\rightarrow\mathbb{R}$ are 
    bounded and continuous functions. Corresponding to every initial distribution $\mu$ 
    on $\mathscr{R}(\mathbb{R}^d)$ with 
    \begin{equation}
        \int\limits_{\mathbb{R}^d} \|x\|^{2m} \mu(dx)<\infty, 
    \end{equation}
    for some $m>1$, there exists a weak solution of the SDE. 
\end{theorem}

\begin{proof}
    For integers $j\geq 0, n\geq 1$ we consider the dyadic rationals $t_j^{(n)}=j2^{-n}$
    and introduce the functions $\psi_n(t)=t_j^{(n)}$; $t\in [t_j^{(n)}, t_{j+1}^{n})$. 
    We define the new coefficients 
    \begin{align}
        b^{(n)}(t, y)\triangleq b(y(\psi_n(t))), \sigma^{(n)}(t,y)\triangleq \sigma(y(\psi(t)));  \\
        0 \leq t < \infty, y\in C[0,\infty)^d, 
    \end{align}
    which are progressively measurable functionals.
    Now let us consider on some probability space $(\Omega, \mathscr{F}, P)$ an $r$-dimensional 
    Brownian motion $W=\{W_t, \mathscr{F}_t^W; 0\leq t<\infty\}$ and an independent random 
    vector $\xi$ with the given initial distribution $\mu$, and let  us construct the 
    filtration $\{\mathscr{F}_t\}$. For each $n\geq 1$, we define the continuous process 
    $X^{(n)}= \{X_t^{(n)}, \mathscr{F}_t; 0\leq t<\infty\}$ by setting $X_0^{(n)}=\xi$ 
    and then recursively:
    \begin{equation}
        X_t^{(n)} = X_{t_j^{(n)}}^{(n)}+b(X_{t_j^{(n)}}^{(n)})(t-t_j^{(n)})+\sigma(X_{t_j^{(n)}}^{(n)})(W_t-W_{t_j^{(n)}}); 
    \end{equation}
    for $j\geq 0, t_j^{(n)}<t\leq t_{j+1}^{(n)}$. Then $X^{(n)}$ solves the functional 
    stochastic integral equation 
    \begin{equation}
        X_t^{(n)}=\xi+\int_{0}^{t} b^{(n)}(s,X^{(n)})ds+\int_{0}^{t} \sigma^{(n)}(s,X^{(n)})dW_s; 
    \end{equation}
    for $0\leq t<\infty$. Fix $0<T<\infty$. We use the inequality 
    \begin{equation}
        \sup\limits_{n\geq 1} E\|X_t^{(n)}-X_s^{(n)}\|^{2m} \leq C(1+E\|\xi\|^{2m})(t-s)^m; 
    \end{equation}
    where $0\leq s<t\leq T$, $C$ is a constant depending only on $m, T,$ the dimension $d$, 
    and the bound on $\|b\|^2+\|\sigma\|^2$. Let $P^{(n)}\triangleq P(X^{(n)})^{-1}; n\geq 1$ 
    be the sequence of probability measures induced on $(C[0,\infty)^d, \mathscr{R}(C[0,\infty)^d))$ 
    by these processes; and this sequence is tight. We may then assert by the Prohorov theorem, 
    relabeling indices if necessary, that the sequence $\{P^{(n)}\}_{n=1}^{\infty}$ converges 
    weakly to a probability measure $P^{\ast}$ on this canonical space. 
\end{proof}

\end{document}